%% file: ArXivPaperTemplate.tex
\documentclass{article}

\def \TRtitle{Hash Function Learning via Codewords}
\def \TRauthors{Yinjie Huang, Michael Georgiopoulos and Georgios C. Anagnostopoulos}
\def \TRemails{yhuang@eecs.ucf.edu, michaelg@ucf.edu and georgio@fit.edu}
\def \TRaffiliations{CL, MG: EE \& CS Dept., University of Central Florida; GCA: ECE Dept., Florida Institute of Technology}
\def \TRkeywords{Hash Function Learning, Codeword, Support Vector Machine}

\usepackage{./sty/options}
\usepackage{./sty/packages}
\usepackage{./sty/macros}
\begin{document}

% Make title pages
\maketitle

% Do not change. %
\ifMakeReviewDraft
	\linenumbers
\fi

\input{Abstract}

% Do not change. %
\vskip 0.5in
\noindent
{\bf Keywords:} \TRkeywords
% /////////////////////////////// //%

% Input sections from separate files. Modify these fields as necessary.
\input{Introduction}

\input{Formulation}

\input{Optimization}

\input{Generalization}

\input{Experiments}

\input{Conclusions}

\input{Acknowledgments}

% The plainrul style will printout the URL field of each bib entry
% and hyperref will create a clickable link.
\bibliographystyle{plainurl}
% Modify the bibliography file name as necessary.
\bibliography{References}

\end{document}

%% file: Abstract.tex
\begin{abstract}
 
In this paper we introduce a novel hash learning framework that has two main distinguishing features, when compared to past approaches. First, it utilizes codewords in the Hamming space as ancillary means to accomplish its hash learning task. These codewords, which are inferred from the data, attempt to capture similarity aspects of the data's hash codes. Secondly and more importantly, the same framework is capable of addressing supervised, unsupervised and, even, semi-supervised hash learning tasks in a natural manner. A series of comparative experiments focused on content-based image retrieval highlights its performance advantages. \footnote{This work has been accepted by ECML/PKDD 2015. Please cite the ECML version of this paper.}

\end{abstract}

%% file: Introduction.tex
% reset all acronyms
\acresetall

\section{Introduction}
\label{sec:Introduction}

With the explosive growth of web data including documents, images and videos, content-based image retrieval (CBIR) has attracted plenty of attention over the past years \cite{Datta2008}. Given a query sample, a typical CBIR scheme retrieves samples stored in a database that are most similar to the query sample. The similarity is gauged in terms of a pre-specified distance metric and the retrieved samples are the nearest neighbors of the query point w.r.t. this metric. However, exhaustively comparing the query sample with every other sample in the database may be computationally expensive in many current practical settings. Additionally, most CBIR approaches may be hindered by the sheer size of each sample; for example, visual descriptors of an image or a video may number in the thousands. Furthermore, storage of these high-dimensional data also presents a challenge.

Considerable effort has been invested in designing hash functions transforming the original data into compact binary codes to reap the benefits of a potentially fast similarity search; note that hash functions are typically designed to preserve certain similarity qualities between the data. For example, approximate nearest neighbors (ANN) search \cite{Torralba2008} using compact binary codes in Hamming space was shown to achieve sub-liner searching time. Storage of the binary code is, obviously, also much more efficient.  

Existing hashing methods can be divided into two categories: \textit{data-independent} and \textit{data-dependent}. The former category does not use a data-driven approach to choose the hash function. For example, Locality Sensitive Hashing (LSH) \cite{Gionis1999} randomly projects and thresholds data into the Hamming space for generating binary codes, where closely located (in terms of Euclidean distances in the data's native space) samples are likely to have similar binary codes. Furthermore, in \cite{Kulis2009}, the authors proposed a method for ANN search using a learned Mahalanobis metric combined with LSH. 

On the other hand, \textit{data-dependent methods} can, in turn, be grouped into supervised, unsupervised and semi-supervised learning paradigms. The bulk of work in data-dependent hashing methods has been performed so far following the supervised learning paradigm. Recent work includes the Semantic Hashing \cite{Salakhutdinov2009}, which designs the hash function using a Restricted Boltzmann Machine (RBM). Binary Reconstructive Embedding (BRE) in \cite{Kulis2009a} tries to minimize a cost function measuring the difference between the original metric distances and the reconstructed distances in the Hamming space. Minimal Loss Hashing (MLH) \cite{Norouzi2011} learns the hash function from pair-wise side information and the problem is formulated based on a bound inspired by the theory of structural Support Vector Machines \cite{Yu2009}. In \cite{Mu2010}, a scenario is addressed, where a small portion of sample pairs are manually labeled as similar or dissimilar and proposes the Label-regularized Max-margin Partition algorithm. Moreover, Self-Taught Hashing \cite{Zhang2010} first identifies binary codes for given documents via unsupervised learning; next, classifiers are trained to predict codes for query documents. Additionally, Fisher Linear Discriminant Analysis (LDA) is employed in \cite{Strecha2012} to embed the original data to a lower dimensional space and hash codes are obtained subsequently via thresholding. Also, Boosting based Hashing is used in \cite{Shakhnarovich2003} and \cite{Baluja2008}, in which a set of weak hash functions are learned according to the boosting framework. In \cite{Li2013}, the hash functions are learned from triplets of side information; their method is designed to preserve the relative relationship reflected by the triplets and is optimized using column generation. Finally, Kernel Supervised Hashing (KSH) \cite{Liu2012} introduces a kernel-based hashing method, which seems to exhibit remarkable experimental results.

As for unsupervised learning, several approaches have been proposed: Spectral Hashing (SPH) \cite{Weiss2008} designs the hash function by using spectral graph analysis with the assumption of a uniform data distribution. \cite{Liu2011} proposed Anchor Graph Hashing (AGH). AGH uses a small-size anchor graph to approximate low-rank adjacency matrices that leads to computational savings. Also, in \cite{Gong2011}, the authors introduce Iterative Quantization, which tries to learn an orthogonal rotation matrix so that the quantization error of mapping the data to the vertices of the binary hypercube is minimized.

To the best of our knowledge, the only approach to date following a semi-supervised learning paradigm is Semi-Supervised Hashing (SSH) \cite{Wang2010} \cite{Wang2012}. The SSH framework minimizes an empirical error using labeled data, but to avoid over-fitting, its model also includes an information theoretic regularizer that utilizes both labeled and unlabeled data.

In this paper we propose \ac{*SHL} (* stands for all three learning paradigms), a novel hash function learning approach, which sets itself apart from past approaches in two major ways. First, it uses a set of Hamming space codewords that are learned during training in order to capture the intrinsic similarities between the data's hash codes, so that same-class data are grouped together. Unlabeled data also contribute to the adjustment of codewords leveraging from the inter-sample dissimilarities of their generated hash codes as measured by the Hamming metric. Due to these codeword-specific characteristics, a major advantage offered by \ac{*SHL} is that it can naturally engage supervised, unsupervised and, even, semi-supervised hash learning tasks using a single formulation. Obviously, the latter ability readily allows \ac{*SHL} to perform transductive hash learning.

In \sref{sec:Formulation}, we provide \ac{*SHL}'s formulation, which is mainly motivated by an attempt to minimize the within-group Hamming distances in the code space between a group's codeword and the hash codes of data. % that either should be similar (because of similar labels), or are de-facto similar (due to particular state of the hash functions).
With regards to the hash functions, \ac{*SHL} adopts a kernel-based approach. The aforementioned formulation eventually leads to a minimization problem over the codewords as well as over the \ac{RKHS} vectors defining the hash functions. A quite noteworthy aspect of the resulting problem is that the minimization over the latter parameters leads to a set of \ac{SVM} problems, according to which each \ac{SVM} generates a single bit of a sample's hash code. In lieu of choosing a fixed, arbitrary kernel function, we use a simple \ac{MKL} approach (\eg\ see \cite{Kloft2011}) to infer a good kernel from the data. We need to note here that Self-Taught Hashing (STH) \cite{Zhang2010} also employs \acp{SVM} to generate hash codes. However, STH differs significantly from \ac{*SHL}; its  unsupervised and supervised learning stages are completely decoupled, while \ac{*SHL} uses a single cost function that simultaneously accommodates both of these learning paradigms. Unlike STH, \acp{SVM} arise naturally from the problem formulation in \ac{*SHL}.

Next, in \sref{sec:Optimization}, an efficient \ac{MM} algorithm is showcased that can be used to optimize \ac{*SHL}'s framework via a \ac{BCD} approach. The first block optimization amounts to training a set of \acp{SVM}, which can be efficiently accomplished by using, for example, \texttt{LIBSVM} \cite{Chang2011}. The second block optimization step addresses the \ac{MKL} parameters, while the third one adjusts the codewords. Both of these steps are computationally fast due to the existence of closed-form solutions.   

Finally, in \sref{sec:Experiments} we demonstrate the capabilities of \ac{*SHL} on a series of comparative experiments. The section emphasizes on supervised hash learning problems in the context of CBIR, since the majority of hash learning approaches address this paradigm. We also included some preliminary transductive hash learning results for \ac{*SHL} as a proof of concept. Remarkably, when compared to other hashing methods on supervised learning hash tasks, \ac{*SHL} exhibits the best retrieval accuracy for all the datasets we considered. Some clues to \ac{*SHL}'s superior performance are provided in \sref{sec:Generalization}.

%% file: Formulation.tex
\section{Formulation}
\label{sec:Formulation}

In what follows, $\left[ \cdot \right]$ denotes the Iverson bracket, \ie, $\left[ \text{predicate} \right] = 1$, if the predicate is true, and $\left[ \text{predicate} \right] = 0$, if otherwise. Additionally, vectors and matrices are denoted in boldface. All vectors are considered column vectors and $\cdot^T$ denotes transposition. Also, for any positive integer $K$, we define $\mathbb{N}_K \triangleq \left\{ 1, \ldots, K \right\}$.

Central to hash function learning is the design of functions transforming data to compact binary codes in a Hamming space to fulfill a given machine learning task. Consider the Hamming space $\mathbb{H}^B \triangleq \left\{-1, 1\right\}^B$, which implies $B$-bit hash codes. \ac{*SHL} addresses multi-class classification tasks with an arbitrary set $\mathcal{X}$ as sample space. It does so by learning a hash function $\mathbf{h}: \mathcal{X} \rightarrow \mathbb{H}^B$ and a set of $G$ labeled codewords $\boldsymbol{\mu}_g, \ g \in \mathbb{N}_G$ (each codeword representing a class), so that the hash code of a labeled sample is mapped close to the codeword corresponding to the sample's class label; proximity is measured via the Hamming distance. Unlabeled samples are also able to contribute to learning both the hash function and the codewords as it will demonstrated in the sequel. Finally, a test sample is classified according to the label of the codeword closest to the sample's hash code.  

In \ac{*SHL}, the hash code for a sample $x \in \mathcal{X}$ is eventually computed as $\mathbf{h}(x) \triangleq \sgn \mathbf{f}(x) \in \mathbb{H}^B$, where the signum function is applied component-wise. Furthermore, $\mathbf{f}(x) \triangleq \left[ f_1(x) \ldots f_B(x) \right]^T$, where $f_b(x) \triangleq \left \langle w_b, \phi(x) \right \rangle_{\mathcal{H}_b} + \beta_b$ with $w_b \in \Omega_{w_b} \triangleq \left\{ w_b \in \mathcal{H}_b: \left\| w_b \right\|_{\mathcal{H}_b} \leq R_b, R_b>0 \right\}$ and $\beta_b \in \mathbb{R}$ for all $b \in \mathbb{N}_B$. In the previous definition, $\mathcal{H}_b$ is a \ac{RKHS} with inner product $\left \langle \cdot, \cdot \right \rangle_{\mathcal{H}_b}$, induced norm $\left\| w_b \right\|_{\mathcal{H}_b} \triangleq \sqrt{ \left \langle w_b, w_b \right \rangle_{\mathcal{H}_b} }$ for all $w_b \in \mathcal{H}_b$, associated feature mapping $\phi_b: \mathcal{X} \rightarrow \mathcal{H}_b$ and reproducing kernel $k_b: \mathcal{X} \times \mathcal{X} \rightarrow \mathbb{R}$, such that $k_b(x,x') = \left \langle \phi_b(x), \phi_b(x') \right \rangle_{\mathcal{H}_b}$ for all $x,x' \in \mathcal{X}$. Instead of a priori selecting the kernel functions $k_b$, \ac{MKL} \cite{Kloft2011} is employed to infer the feature mapping for each bit from the available data. In specific, it is assumed that each \ac{RKHS} $\mathcal{H}_b$ is formed as the direct sum of $M$ common, pre-specified \acp{RKHS} $\mathcal{H}_m$, \ie, $\mathcal{H}_b = \bigoplus_{m} \sqrt{\theta_{b,m}} \mathcal{H}_m$, where $\boldsymbol{\theta}_b \triangleq \left[ \theta_{b,1} \ldots \theta_{b,M} \right]^T \in \Omega_{\theta} \triangleq \left\{ \boldsymbol{\theta} \in \mathbb{R}^M: \boldsymbol{\theta} \succeq \mathbf{0}, \left\| \boldsymbol{\theta} \right\|_p \leq 1, p \geq 1 \right\}$, $\succeq$ denotes the component-wise $\geq$ relation, $\left\| \cdot \right\|_p$ is the usual $l_p$ norm in $\mathbb{R}^M$ and $m$ ranges over $\mathbb{N}_M$. Note that, if each preselected \ac{RKHS} $\mathcal{H}_m$ has associated kernel function $k_m$, then it holds that $k_b(x,x') = \sum_{m} \theta_{b,m} k_m(x,x')$ for all $x,x' \in \mathcal{X}$.  

Now, assume a training set of size $N$ consisting of labeled and unlabeled samples and let $\mathcal{N}_L$ and $\mathcal{N}_U$ be the index sets for these two subsets respectively. Let also $l_n$ for $n \in \mathcal{N}_L$ be the class label of the $n^{th}$ labeled sample. By adjusting its parameters, which are collectively denoted as $\boldsymbol{\omega}$, \ac{*SHL} attempts to reduce the distortion measure

\begin{align}
\label{eq:1}
	E(\boldsymbol{\omega}) \triangleq \sum_{n \in \mathcal{N}_L} d\left( \mathbf{h}(x_n), \boldsymbol{\mu}_{l_n} \right) + \sum_{n \in \mathcal{N}_U} \min_{g} d\left( \mathbf{h}(x_n), \boldsymbol{\mu}_{g} \right)
\end{align}

\noindent
where $d$ is the Hamming distance defined as $d(\mathbf{h}, \mathbf{h}') \triangleq \sum_{b} \left[ h_b \neq h'_b \right]$. However, the distortion $E$ is difficult to directly minimize. As it will be illustrated further below, an upper bound $\bar{E}$ of $E$ will be optimized instead.

In particular, for a hash code produced by \ac{*SHL}, it holds that $d\left( \mathbf{h}(x), \boldsymbol{\mu} \right) = \\ \sum_b \left[ \mu_b f_b(x) < 0 \right]$. If one defines $\bar{d}\left( \mathbf{f}, \boldsymbol{\mu} \right) \triangleq \sum_b \hinge{1 - \mu_b f_b}$, where $\hinge{u} \triangleq \max \left\{ 0, u \right\}$ is the hinge function, then $d\left( \sgn \mathbf{f}, \boldsymbol{\mu} \right) \leq \bar{d}\left( \mathbf{f}, \boldsymbol{\mu} \right)$ holds for every $\mathbf{f} \in \mathbb{R}^B$ and any $\boldsymbol{\mu} \in \mathbb{H}^B$. Based on this latter fact, it holds that

\begin{align}
	\label{eq:2}
	E(\boldsymbol{\omega}) & \leq \bar{E}(\boldsymbol{\omega}) \triangleq \sum_{g} \sum_{n} \gamma_{g,n} \bar{d}\left( \mathbf{f}(x_n), \boldsymbol{\mu}_g \right)
	\intertext{where}
	\label{eq:3}
	\gamma_{g,n} & \triangleq 
		\begin{cases}
			\left[ g = l_n \right] & n \in \mathcal{N}_L \\
			\left[ g = \arg \min_{g'} \bar{d}\left( \mathbf{f}(x_n), \boldsymbol{\mu}_{g'} \right)  \right] & n \in \mathcal{N}_U
		\end{cases}
\end{align}

\noindent
It turns out that $\bar{E}$, which constitutes the model's loss function, can be efficiently minimized by a three-step algorithm, which delineated in the next section. 

%% file: Optimization.tex
\section{Learning Algorithm}
\label{sec:Optimization}

The next proposition allows us to minimize $\bar{E}$ as defined in \eref{eq:2} via a \ac{MM} approach \cite{Hunter2004}, \cite{Hunter2000}. 

%%%%%%%%%% P R O P O S I T I O N  %%%%%%%%%%%%%%%%%%%%
\begin{proposition}
\label{prop:1}
For any \ac{*SHL} parameter values $\boldsymbol{\omega}$ and $\boldsymbol{\omega}'$, it holds that

\begin{align}
	\label{eq:4}
	\bar{E}(\boldsymbol{\omega}) & \leq \bar{E}(\boldsymbol{\omega} | \boldsymbol{\omega}') \triangleq \sum_{g} \sum_{n} \gamma'_{g,n} \bar{d}\left( \mathbf{f}(x_n), \boldsymbol{\mu}_g \right)
	\intertext{where the primed quantities are evaluated on $\boldsymbol{\omega}'$ and}
	\label{eq:5}
	\gamma'_{g,n} & \triangleq 
		\begin{cases}
			\left[ g = l_n \right] & n \in \mathcal{N}_L \\
			\left[ g = \arg \min_{g'} \bar{d}\left( \mathbf{f}'(x_n), \boldsymbol{\mu}'_{g'} \right)  \right] & n \in \mathcal{N}_U
		\end{cases}
\end{align}

\noindent
Additionally, it holds that $\bar{E}(\boldsymbol{\omega} | \boldsymbol{\omega}) = \bar{E}(\boldsymbol{\omega})$ for any $\boldsymbol{\omega}$. In summa, $\bar{E}(\cdot | \cdot)$ majorizes $\bar{E}(\cdot)$.

\end{proposition} %%%%%%%%%%%%%%%%%%%%%%%%%%%%%%%%%%%%%%%%

\noindent
Its proof is relative straightforward and is based on the fact that for any value of $\gamma'_{g,n} \in \left\{0, 1\right\}$ other than $\gamma_{g,n}$ as defined in \eref{eq:3}, the value of $\bar{E}(\boldsymbol{\omega} | \boldsymbol{\omega}')$ can never be less than $\bar{E}(\boldsymbol{\omega} | \boldsymbol{\omega}) = \bar{E}(\boldsymbol{\omega})$.

The last proposition gives rise to a \ac{MM} approach, where $\boldsymbol{\omega}'$ are the current estimates of the model's parameter values and $\bar{E}(\boldsymbol{\omega} | \boldsymbol{\omega}')$ is minimized with respect to $\boldsymbol{\omega}$ to yield improved estimates $\boldsymbol{\omega}^*$, such that $\bar{E}(\boldsymbol{\omega}^*) \leq \bar{E}(\boldsymbol{\omega}')$. This minimization can be achieved via a \ac{BCD}.%, as is argued based on the next proposition.

%%%%%%%%%% P R O P O S I T I O N  %%%%%%%%%%%%%%%%%%%%
\begin{proposition}
\label{prop:2}

Minimizing $\bar{E}(\cdot | \boldsymbol{\omega}')$ with respect to the Hilbert space vectors, the offsets $\beta_p$ and the \ac{MKL} weights $\boldsymbol{\theta}_b$, while regarding the codeword parameters as constant, one obtains the following $B$ independent, equivalent problems:

\begin{align}
	\label{eq:6} 
	\underset{\underset{\beta_b \in \mathbb{R}, \boldsymbol{\theta}_b \in \Omega_{\theta}, \mu_{g,b} \in \mathbb{H}}{w_{b,m} \in \mathcal{H}_m, m \in \mathbb{N}_M}}{\inf}
	C \sum_g \sum_n \gamma'_{g,n} \hinge{ 1 - \mu_{g,b} f_b(x_n)  } \nonumber \\  +  \frac{1}{2} \sum_m \frac{ \left\| w_{b,m} \right\|^2_{\mathcal{H}_m} }{\theta_{b,m}} \ \ \ b \in \mathbb{N}_B
\end{align}

\noindent
where $f_b(x) = \sum_m \left \langle w_{b,m} , \phi_m(x)  \right \rangle_{\mathcal{H}_m} + \beta_b$ and $C>0$ is a regularization constant.

\end{proposition} %%%%%%%%%%%%%%%%%%%%%%%%%%%%%%%%%%%%%%%%

The proof of this proposition hinges on replacing the (independent) constraints of the Hilbert space vectors with equivalent regularization terms and, finally, performing the substitution $w_{b,m} \gets \sqrt{\theta_{b,m}} w_{b,m}$ as typically done in such \ac{MKL} formulations (\eg\ see \cite{Kloft2011}). Note that \pref{eq:6} is jointly convex with respect to all variables under consideration and, under closer scrutiny, one may recognize it as a binary \ac{MKL} \ac{SVM} training problem, which will become more apparent shortly.  

\textbf{First block minimization:} By considering $w_{b,m}$ and $\beta_b$ for each $b$ as a single block, instead of directly minimizing \pref{eq:6}, one can instead maximize the following problem:

%%%%%%%%%% P R O P O S I T I O N  %%%%%%%%%%%%%%%%%%%%
\begin{proposition}
\label{prop:2.5}

The dual form of \pref{eq:6} takes the form of

\begin{align}
\label{eq:7}
\underset{\boldsymbol{\alpha}_b \in \Omega_{a_b}}{\sup} & \ \ \boldsymbol{\alpha}_b^T	\boldsymbol{1}_{NG} - \frac{1}{2} \boldsymbol{\alpha}_b^T \mathbf{D}_b [(\boldsymbol{1}_G \boldsymbol{1}_G^T) \otimes \mathbf{K}_b] \mathbf{D}_b \boldsymbol{\alpha}_b \ \ \ b \in \mathbb{N}_B
\end{align}

\noindent
where $\mathbf{1}_K$ stands for the all ones vector of $K$ elements ($K \in \mathbb{N}$), $\boldsymbol{\mu}_b \triangleq \left[ \mu_{1,b} \ldots \mu_{G,b} \right]^T$, $\mathbf{D}_b \triangleq \diag{\boldsymbol{\mu}_b \otimes \mathbf{1}_N}$, $\mathbf{K}_b \triangleq \sum_m \theta_{b,m} \mathbf{K}_m$, where $\mathbf{K}_m$ is the data's $m^{th}$ kernel matrix, $\Omega_{a_b} \triangleq \left\{ \boldsymbol{\alpha} \in \mathbb{R}^{NG}: \boldsymbol{\alpha}_b^T (\boldsymbol{\mu}_b \otimes \boldsymbol{1}_N) = 0, \boldsymbol{0} \preceq \boldsymbol{\alpha}_b \preceq C \boldsymbol{\gamma}'  \right\}$ \\ and $\boldsymbol{\gamma}' \triangleq \left[ \gamma'_{1,1}, \ldots, \gamma'_{1,N}, \gamma'_{2,1}, \ldots, \gamma'_{G,N} \right]^T$. 
\end{proposition} %%%%%%%%%%%%%%%%%%%%%%%%%%%%%%%%%%%%%%%%

\begin{proof}
After eliminating the hinge function in \pref{eq:6} with the help of slack variables $\xi^b_{g,n}$, we obtain the following problem for the first block minimization:

\begin{align}
\label{eq:s1}
\underset{\xi^b_{g,n}}{\underset{w_{b,m}, \beta_b }{min}} & \ \ C\sum_g \sum_n \gamma'_{g,n} \xi^b_{g,n} + \frac{1}{2} \sum_m \frac{\left\| w_{b,m} \right\|^2_{\mathcal{H}_m}}{\theta_{b,m}} \nonumber \\
\text{s.t.} & \ \ \xi^b_{g,n} \geq 0 \nonumber \\
	 & \ \ \xi^b_{g,n} \geq 1-(\sum_m \left \langle w_{b,m} , \phi_m(x)  \right \rangle_{\mathcal{H}_m} + \beta_b)\mu_{g,b}
\end{align}

\noindent Due to the Representer Theorem (\eg, see \cite{Scholkopf2001}), we have that

\begin{align}
\label{eq:s2}
w_{b,m} = \theta_{b,m}\sum_n \eta_{b,n} \phi_m(x_n)
\end{align}

\noindent where $n$ is the training sample index. By defining %$\mathbb{I}_b \triangleq [\xi^b_{g,n}] \in \mathbb{R}^{N \times G}$, 
$\boldsymbol{\xi}_b \in \mathbb{R}^{RG}$ to be the vector containing all $\xi^b_{g,n}$'s, %\triangleq vec[\mathbb{I}_b] \in \mathbb{R}^{RG}$, 
$\boldsymbol{\eta}_b \triangleq [\eta_{b,1},\eta_{b,2},...,\eta_{b,N}]^T \in \mathbb{R}^N$ and $\boldsymbol{\mu}_b \triangleq [\mu_{1,b}, \mu_{2,b},...,\mu_{G,b}]^T \in \mathbb{R}^{G}$, the vectorized version of \pref{eq:s1} in light of \eref{eq:s2} becomes

\begin{align}
\label{eq:s3}
\underset{\boldsymbol{\eta}_b, \boldsymbol{\xi}_b, \beta_b}{min} & \ \ C \boldsymbol{\gamma}' \boldsymbol{\xi}_b + \frac{1}{2} \boldsymbol{\eta}^T_b \mathbf{K}_b \boldsymbol{\eta}_b\nonumber \\
s.t. & \ \ \boldsymbol{\xi}_b \succeq \boldsymbol{0} \nonumber \\
	 & \ \ \boldsymbol{\xi}_b \succeq \boldsymbol{1}_{NG}-(\boldsymbol{\mu}_b \otimes \mathbf{K}_b)\boldsymbol{\eta}_b-(\boldsymbol{\mu}_b \otimes \boldsymbol{1}_N)\beta_b
\end{align}

\noindent 
where $\boldsymbol{\gamma}'$ and $\mathbf{K}_b$ are defined in \propref{prop:2.5}. %Take the Lagrangian $\mathcal{L}$ and its derivatives, we have the following relations, here $\boldsymbol{\alpha}_b$ and $\boldsymbol{\lambda}_b$ are Lagrangian multipliers for the two constraints:
From the previous problem's Lagrangian $\mathcal{L}$, one obtains

\begin{align}
\label{eq:s4}
\frac{\partial \mathcal{L}}{\partial \boldsymbol{\xi}_b} = \boldsymbol{0} & \Rightarrow \begin{cases}
\boldsymbol{\lambda}_b = C\boldsymbol{\gamma}'-\boldsymbol{\alpha}_b  \\ 
 \boldsymbol{0} \preceq \boldsymbol{\alpha}_b \preceq C \boldsymbol{\gamma}'
\end{cases} \\
\label{eq:s5}
\frac{\partial \mathcal{L}}{\partial \beta_b} = 0 & \Rightarrow \boldsymbol{\alpha}_b^T(\boldsymbol{\mu}_b \otimes\boldsymbol{1}_N) = 0 \\
\label{eq:s6}
\frac{\partial \mathcal{L}}{\partial \boldsymbol{\eta}_b} = \boldsymbol{0} & \overset{\exists \mathbf{K}_b^{-1}}{\Rightarrow} \boldsymbol{\eta}_b = \mathbf{K}_b^{-1}(\boldsymbol{\mu}_b \otimes \mathbf{K}_b)^T \boldsymbol{\alpha}_b  
\end{align}

\noindent 
where $\boldsymbol{\alpha}_b$ and $\boldsymbol{\lambda}_b$ are the dual variables for the two constraints in \pref{eq:s3}. Utilizing \eref{eq:s4}, \eref{eq:s5} and \eref{eq:s6}, the quadratic term of the dual problem becomes 

\begin{align}
\label{eq:s7}
& (\boldsymbol{\mu}_b \otimes \mathbf{K}_b) \mathbf{K}_b^{-1} (\boldsymbol{\mu}_b^T \otimes \mathbf{K}_b)= \nonumber \\
& = (\boldsymbol{\mu}_b \otimes \mathbf{K}_b) (1 \otimes \mathbf{K}_b^{-1}) (\boldsymbol{\mu}_b^T \otimes \mathbf{K}_b) \nonumber \\
& = (\boldsymbol{\mu}_b \otimes \mathbf{I}_{N \times N})(\boldsymbol{\mu}_b^T \otimes \mathbf{K}_b) \nonumber \\
& = (\boldsymbol{\mu}_b \boldsymbol{\mu}_b^T) \otimes \mathbf{K}_b
\intertext{\eref{eq:s7} can be further manipulated as}
%\end{align}
%
%\noindent \eref{eq:s7} can be further manipulated as
%
%\begin{align}
\label{eq:s8}
& (\boldsymbol{\mu}_b \boldsymbol{\mu}_b^T) \otimes \mathbf{K}_b= \nonumber \\
& = [(\diag{\boldsymbol{\mu}_b} \boldsymbol{1}_G)(\diag{\boldsymbol{\mu}_b} \boldsymbol{1}_G)^T] \otimes \mathbf{K}_b \nonumber \\
& = [\diag{\boldsymbol{\mu}_b}(\boldsymbol{1}_G \boldsymbol{1}_G^T)\diag{\boldsymbol{\mu}_b}] \otimes [\mathbf{I}_N \mathbf{K}_b \mathbf{I}_N] \nonumber \\
& = [\diag{\boldsymbol{\mu}_b} \otimes \mathbf{I}_N][(\boldsymbol{1}_G \boldsymbol{1}_G^T) \otimes \mathbf{K}_b][\diag{\boldsymbol{\mu}_b} \otimes \mathbf{I}_N] \nonumber \\
& = [\diag{\boldsymbol{\mu}_b \otimes \boldsymbol{1}_N}][(\boldsymbol{1}_G \boldsymbol{1}_G^T) \otimes \mathbf{K}_b][\diag{\boldsymbol{\mu}_b \otimes \boldsymbol{1}_N}] \nonumber \\
& = \mathbf{D}_b [(\boldsymbol{1}_G \boldsymbol{1}_G^T) \otimes \mathbf{K}_b] \mathbf{D}_b
\end{align}

\noindent The first equality stems from the identity $\diag{\boldsymbol{v}}\boldsymbol{1} = \boldsymbol{v}$ for any vector $\boldsymbol{v}$, while the third one stems form the mixed-product property of the Kronecker product. Also, the identity $\diag{\boldsymbol{v} \otimes \boldsymbol{1}} = \diag{\boldsymbol{v}} \otimes \mathbf{I}$ yields the fourth equality. Note that $\mathbf{D}_b$ is defined as in \propref{prop:2.5}. Taking into account \eref{eq:s7} and \eref{eq:s8}, we reach the dual form stated in \propref{prop:2.5}.
\end{proof}

Given that $\gamma'_{g,n} \in \left\{ 0, 1 \right\}$, one can easily now recognize that \pref{eq:7} is an \ac{SVM} training problem, which can be conveniently solved using software packages such as \texttt{LIBSVM}. After solving it, obviously one can compute the quantities $\left \langle w_{b,m} , \phi_m(x)  \right \rangle_{\mathcal{H}_m}$, $\beta_b$ and $\left\| w_{b,m} \right\|^2_{\mathcal{H}_m}$, which are required in the next step. 

\textbf{Second block minimization:} Having optimized over the \ac{SVM} parameters, one can now optimize the cost function of \pref{eq:6} with respect to the \ac{MKL} parameters $\boldsymbol{\theta}_b$ as a single block using the closed-form solution mentioned in Prop. 2 of \cite{Kloft2011} for $p>1$ and which is given next.

\begin{align}
	\label{eq:8}
	\theta_{b,m} = \frac{\left \| w_{b,m} \right \|^{\frac{2}{p+1}}_{\mathcal{H}_m}}{ \left( \sum_{m'} \left \| w_{b,m'} \right \|^{\frac{2p}{p+1}}_{\mathcal{H}_{m'}}  \right)^{\frac{1}{p}}}, \ \ \ m \in \mathbb{N}_M, b \in \mathbb{N}_B.
\end{align}

\textbf{Third block minimization:} Finally, one can now optimize the cost function of \pref{eq:6} with respect to the codewords by mere substitution as shown below.

\begin{align}
	\label{eq:9}
	\underset{\mu_{g,b} \in \mathbb{H}}{\inf} \ \ \sum_n \gamma_{g,n} \hinge{1 - \mu_{g,b} f_b(x_n)} \ \ \ g \in \mathbb{N}_G, b \in \mathbb{N}_B
\end{align}

On balance, as summarized in \aref{alg1}, for each bit, the combined \ac{MM}/\ac{BCD} algorithm consists of one \ac{SVM} optimization step, and two fast steps to optimize the \ac{MKL} coefficients and codewords respectively. Once all model parameters $\boldsymbol{\omega}$ have been computed in this fashion, their values become the current estimate (\ie, $\boldsymbol{\omega}' \gets \boldsymbol{\omega}$ ), the $\gamma_{g,n}$'s are accordingly updated and the algorithm continues to iterate until convergence is established\footnote{A \texttt{MATLAB\textsuperscript{\circledR}} implementation of our framework is available at \\  \href{https://github.com/yinjiehuang/StarSHL}{https://github.com/yinjiehuang/StarSHL}}. Based on \texttt{LIBSVM}, which provides $\mathcal{O}(N^3)$ complexity \cite{List2009}, our algorithm offers the complexity $\mathcal{O}(BN^3)$ per iteration , where $B$ is the code length and $N$ is the number of instances.   
 
 \begin{algorithm}[tb]
 \caption{Optimization of \pref{eq:6}}
 \label{alg1}
 \begin{algorithmic}
 \STATE {\bfseries Input:} Bit Length $B$, Training Samples $X$ containing labeled or unlabled data.
 \STATE {\bfseries Output:} $\boldsymbol{\omega}$.
 \STATE 1. Initialize $\boldsymbol{\omega}$.
 \STATE 2. {\bfseries While Not Converged}
 \STATE 3. \quad {\bfseries For each bit} 
 \STATE 4. \quad \quad $\gamma_{g, n}' \leftarrow \eref{eq:5}$.
 \STATE 5. \quad \quad Step 1: $w_{b,m} \leftarrow \eref{eq:7}$.
 \STATE 6. \quad \quad \quad \quad \quad \ $\beta_b \leftarrow \eref{eq:7}$.
 \STATE 7. \quad \quad Step 2: Compute $\left\| w_{b,m} \right\|^2_{\mathcal{H}_m}$.
 \STATE 8. \quad \quad \quad \quad \quad \ $\theta_{b, m} \leftarrow \eref{eq:8}$.
 \STATE 9. \quad \quad Step 3: $\mu_{g,b} \leftarrow \eref{eq:9}$. 
 \STATE 10. \quad {\bfseries End For}
 \STATE 11. {\bfseries End While}
 \STATE 12. Output $\boldsymbol{\omega}$.
 \end{algorithmic}
 \end{algorithm}

%% file: Generalization.tex
\section{Insights to Generalization Performance}
\label{sec:Generalization}

The superior performance of \ac{*SHL} over other state-of-the-art hash function learning approaches featured in the next section can be explained to some extend by noticing that \ac{*SHL} training attempts to minimize the normalized (by $B$) expected Hamming distance of a labeled sample to the correct codeword, which is demonstarted next. We constrain ourselves to the case, where the training set consists only of labeled samples (\ie, $N = \mathcal{N}_L$, $\mathcal{N}_U = 0$) and, for reasons of convenience, to a single-kernel learning scenario, where each code bit is associated to its own feature space $\mathcal{H}_b$ with corresponding kernel function $k_b$. Also, due to space limitations, we provide the next result without proof.

%%%%%%%%%% L E M M A %%%%%%%%%%%%%%%%%
\begin{lemma}
\label{lemma:1}
Let $\mathcal{X}$ be an arbitrary set, $\mathcal{F} \triangleq \{ \mathbf{f}: x \mapsto \mathbf{f}(x) \in \mathbb{R}^B, \ x \in \mathcal{X} \}$, $\Psi: \mathbb{R}^B \rightarrow \mathbb{R}$ be L-Lipschitz continuous w.r.t $\left \| \cdot \right \|_1$, then

\begin{align}
\label{eq:4-1}
\hat{\Re}_N \left( \Psi \circ \mathcal{F} \right) \leq L\hat{\Re}_N \left( \left \| \mathcal{F} \right \|_1 \right)
\end{align}

\noindent 
where $\circ$ stands for function composition, $\hat{\Re}_N(\mathcal{G}) \triangleq \frac{1}{N} \Es{\boldsymbol{\sigma}} { \sup_{g \in \mathcal{G}} \sum_n \sigma_n g(x_n, l_n)}$ is the empirical Rademacher complexity of a set $\mathcal{G}$ of functions, $\{ x_n, l_n\}$ are i.i.d. samples and $\sigma_n$ are i.i.d random variables taking values with $Pr\{ \sigma_n = \pm 1\} = \frac{1}{2}$.
\end{lemma}

\noindent
To show the main theoretical result of our paper with the help of the previous lemma, we will consider the sets of functions

\begin{align}
\label{eq:4-2}
\bar{\mathcal{F}} \triangleq & \{ \mathbf{f}: x \mapsto [f_1(x), ..., f_B(x)]^T, f_b \in \mathcal{F}_b, b \in \mathbb{N}_B\} \\
\label{eq:4-3}
\mathcal{F}_b \triangleq & \{ f_b: x \mapsto \left\langle w_b, \phi_b(x) \right\rangle_{\mathcal{H}_b} + \beta_b, \  \beta_b \in \mathbb{R} \ \text{s.t.} \ |\beta_b| \leq M_b,\nonumber \\ 
& w_b \in \mathcal{H}_b \ \text{s.t.} \ \left \| w_b \right \|_{\mathcal{H}_b} \leq R_b, \ b \in \mathbb{N}_B \}
\end{align}

\noindent
\begin{theorem}
\label{theorem1}
Assume reproducing kernels of $\{ \mathcal{H}_b \}^B_{b=1} \ \text{s.t.} \ k_b(x,x') \leq r^2, \ \forall x, x' \in \mathcal{X}$. Then for a fixed value of $\rho > 0$, for any $\mathbf{f} \in \bar{\mathcal{F}}$, any $\{ \boldsymbol{\mu}_l \}^G_{l=1}, \ \boldsymbol{\mu}_l \in \mathbb{H}^B$ and any $\delta > 0$, with probability $1 - \delta$, it holds that:

\begin{align}
\label{eq:4-3}
er\left(\mathbf{f}, \boldsymbol{\mu}_l \right) \leq \hat{er} \left ( \mathbf{f}, \boldsymbol{\mu}_l \right) + \frac{2r}{\rho B \sqrt{N}}\sum_bR_b + \sqrt{\frac{\log \left( \frac{1}{\delta} \right)}{2 N}}
\end{align}

\noindent
where $er \left( \mathbf{f}, \boldsymbol{\mu}_l \right) \triangleq \frac{1}{B} \mathbb{E} \{ d \left( \sgn \left( \mathbf{f}(x), \boldsymbol{\mu}_l \right) \right) \}$, $l \in \mathbb{N}_G$ is the true label of $x \in \mathcal{X}$, $\hat{er}\left( \mathbf{f}, \boldsymbol{\mu}_l \right) \triangleq \frac{1}{N B} \sum_{n, b} Q_{\rho} \left( f_b(x_n)\mu_{l_n, b} \right)$, where $Q_{\rho}(u) \triangleq \min\left\{1, \max\left\{0, 1 - \frac{u}{\rho} \right\}  \right\}$.
\end{theorem}

\begin{proof}
Notice that

\begin{align}
\label{eq:4-4}
& \frac{1}{B} d \left( \sgn \left( \mathbf{f}(x), \boldsymbol{\mu}_l \right) \right) = \frac{1}{B} \sum_b \left[ f_b(x) \mu_{l,b} < 0 \right] \leq \frac{1}{B} \sum_b Q_{\rho} \left( f_b(x) \mu_{l,b} \right) \nonumber \\
\Rightarrow & \  \mathbb{E} \left\{ \frac{1}{B} d\left( \sgn \left( \mathbf{f}(x), \boldsymbol{\mu}_l \right) \right) \right\} \leq \mathbb{E} \left\{ \frac{1}{B} \sum_b Q_{\rho} \left( f_b(x) \mu_{l,b} \right) \right\}
\end{align}

\noindent
Consider the set of functions

\begin{align}
\Psi \triangleq \{ \psi:(x,l) \mapsto \frac{1}{B} \sum_b Q_{\rho} \left( f_b(x) \mu_{l,b} \right), \mathbf{f} \in \bar{\mathcal{F}}, \mu_{l,b} \in \{ \pm 1 \}, l \in \mathbb{N}_G, b \in \mathbb{N}_B  \} \nonumber
\end{align}

\noindent
Then from Theorem 3.1 of \cite{Mohri2012} and \eref{eq:4-4}, $\forall \psi \in \Psi$, $\exists \delta > 0$, with probability at least $1 - \delta$, we have:

\begin{align}
\label{eq:4-5}
er \left( \mathbf{f}, \boldsymbol{\mu}_l \right) \leq \hat{er} \left( \mathbf{f}, \boldsymbol{\mu}_l \right) + 2 \Re_N(\Psi) + \sqrt{\frac{\log \left( \frac{1}{\delta} \right)}{2 N}}
\end{align}

\noindent
where $\Re_N \left( \Psi \right) $ is the Rademacher complexity of $\Psi$. From \lemmaref{lemma:1}, the following inequality between empirical Rademacher complexities is obtained

\begin{align}
\label{eq:4-6}
\hat{\Re}_N \left( \Psi \right) \leq \frac{1}{B \rho} \hat{\Re}_N \left( \left\| \bar{\mathcal{F}_{\mu}} \right\|_1 \right)
\end{align}

\noindent
where $\bar{\mathcal{F}_{\mu}} \triangleq \{ (x, l) \mapsto [f_1(x)\mu_{l,1},...,f_B(x)\mu_{l, B}]^T, \ f \in \bar{\mathcal{F}} \ and \ \mu_{l,b} \in \{ \pm 1\} \}$. The right side of \eref{eq:4-6} can be upper-bounded as follows

\begin{align}
\label{eq:4-7}
& \hat{\Re}_N \left( \left \| \bar{\mathcal{F}}_{\mu} \right \|_1 \right) \nonumber 
= \frac{1}{N} \Es{\boldsymbol{\sigma}} {\sup_{ \mathbf{f} \in \bar{\mathcal{F}}, \{ \boldsymbol{\mu}_{l_n} \} \in \mathbb{H}^B} \sum_n \sigma_n \sum_b |\mu_{l_n, b} f_b(x_n)|} \nonumber \\
&
 = \frac{1}{N} \Es{\boldsymbol{\sigma}} {\sup_{ \mathbf{f} \in \bar{\mathcal{F}}} \sum_n \sigma_n \sum_b | f_b(x_n)|} \nonumber \\
& = \frac{1}{N} \Es{\boldsymbol{\sigma}} {\sup_{ \omega_b \in \mathcal{H}_b, \left \| \omega_b \right \|_{\mathcal{H}_b} \leq R_b, | \beta_b | \leq M_b } \sum_n \sigma_n \sum_b |\left\langle w_b, \phi_b(x) \right\rangle_{\mathcal{H}_b} + \beta_b|} \nonumber \\
& = \frac{1}{N} \Es{\boldsymbol{\sigma}} {\sup_{ \omega_b \in \mathcal{H}_b, \left \| \omega_b \right \|_{\mathcal{H}_b} \leq R_b, | \beta_b | \leq M_b } \sum_n \sigma_n \sum_b |\left\langle w_b, \sgn(\beta_b) \phi_b(x) \right\rangle_{\mathcal{H}_b} + |\beta_b||} \nonumber \\
& 
= \frac{1}{N} \Es{\boldsymbol{\sigma}} {\sup_{| \beta_b | \leq M_b } \sum_b [R_b \sqrt{\boldsymbol{\sigma}^T K_b \boldsymbol{\sigma}} + |\beta_b| \sum_n \sigma_n]} \nonumber \\
&
= \frac{1}{N} \Es{\boldsymbol{\sigma}} { \sum_b R_b \sqrt{\boldsymbol{\sigma}^T K_b \boldsymbol{\sigma}} }  
\overset{\text{Jensen's Ineq.}}{\leq} \frac{1}{N} \sum_b R_b\sqrt{\Es{\boldsymbol{\sigma}}{\boldsymbol{\sigma}^T K_b \boldsymbol{\sigma}}} \nonumber \\
& = \frac{1}{N} \sum_b R_b \sqrt{\mathrm{trace}\{ K_b \}} \leq \frac{r}{\sqrt{N}} \sum_b R_b
\end{align}

%\noindent
%where in step $(S1)$, $\mu_{l_n, b}$ is eliminated due to the absolute values, Step $(S2)$ is maximization over $\omega_b$. Step $(S3)$ is obtained because $\Es{\boldsymbol{\sigma}}{\sum_n \sigma_n} = 0$. Finally, step $(S4)$ is Jenson's inequality.

\noindent
From \eref{eq:4-6} and \eref{eq:4-7} we obtain $\hat{\Re}_N\left( \Psi  \right) \leq \frac{r}{\rho B \sqrt{N}}\sum_b R_b$. Since $\Re_N \left( \Psi \right) \triangleq \Es{s}{\hat{\Re}_N \left( \Psi \right)}$, where $\mathbb{E}_s$ is the expectation over the samples, we have

\begin{align}
\label{eq:4-8}
\Re_N \left( \Psi \right) \leq \frac{r}{\rho B \sqrt{N}} \sum_b R_b 
\end{align}

\noindent
The final result is obtained by combining \eref{eq:4-5} and \eref{eq:4-8}.
\end{proof}

It can be observed that, minimizing the loss function of \pref{eq:6}, in essence, also reduces the bound of \eref{eq:4-3}. This tends to cluster same-class hash codes around the correct codeword. Since samples are classified according to the label of the codeword that is closest to the sample's hash code, this process may lead to good recognition rates, especially when the number of samples $N$ is high, in which case the bound becomes tighter.

%% file: Experiments.tex
\section{Experiments}
\label{sec:Experiments}

\subsection{Supervised Hash Learning Results}
\label{compareison}

\begin{figure*}[htb]
\vskip 0.2in
\begin{center}
\centerline{\includegraphics[width=\textwidth]{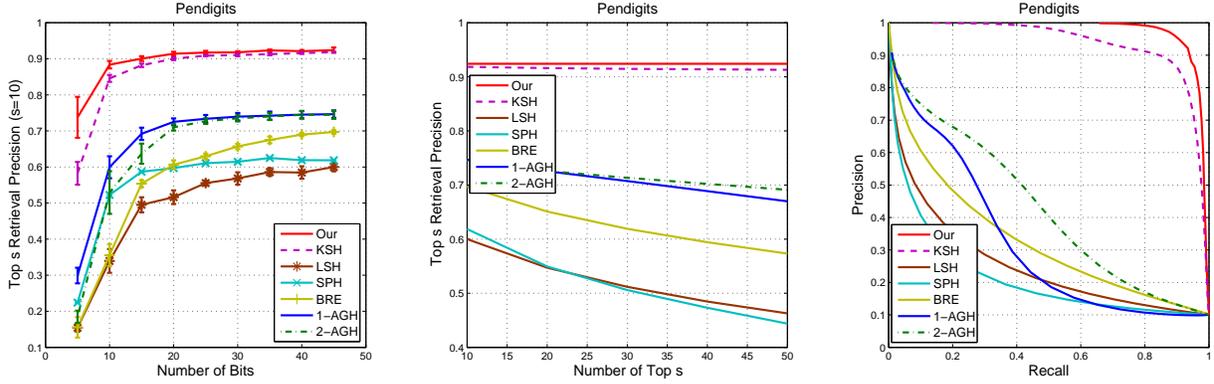}}
\caption{The top $s$ retrieval results and Precision-Recall curve on \textit{Pendigits} dataset over \ac{*SHL} and $6$ other hashing algorithms. (view in color)}
\label{figure2}
\end{center}
\vskip -0.2in
\end{figure*} 

In this section, we compare \ac{*SHL} to other state-of-the-art hashing algorithms: Kernel Supervised Learning (KSH) \cite{Liu2012}, Binary Reconstructive Embedding (BRE) \cite{Kulis2009a}, single-layer Anchor Graph Hashing (1-AGH) and its two-layer version (2-AGH) \cite{Liu2011}, Spectral Hashing (SPH) \cite{Weiss2008} and Locality-Sensitive Hashing (LSH) \cite{Gionis1999}. 

\begin{figure*}[htb]
	\vskip 0.2in
	\begin{center}
		\centerline{\includegraphics[width=\textwidth]{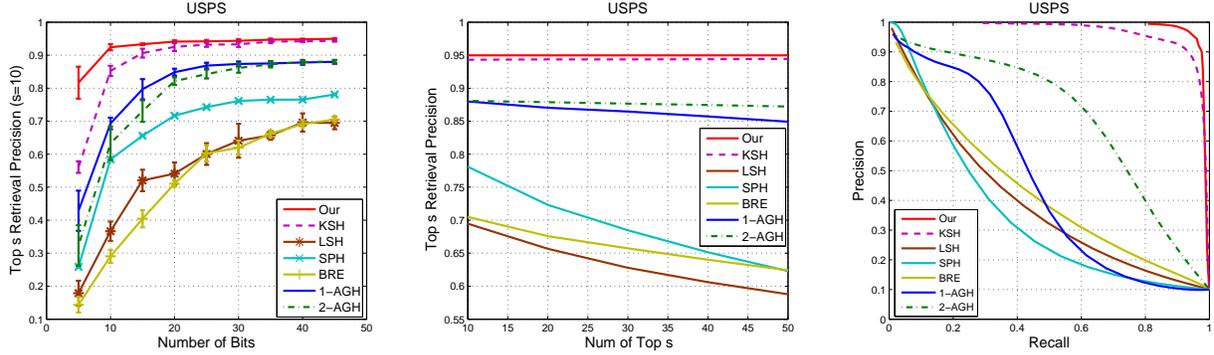}}
		\caption{The top $s$ retrieval results and Precision-Recall curve on \textit{USPS} dataset over \ac{*SHL} and $6$ other hashing algorithms. (view in color)}
		\label{figure3}
	\end{center}
	\vskip -0.2in
\end{figure*} 

\begin{figure*}[htb]
	\vskip 0.2in
	\begin{center}
		\centerline{\includegraphics[width=\textwidth]{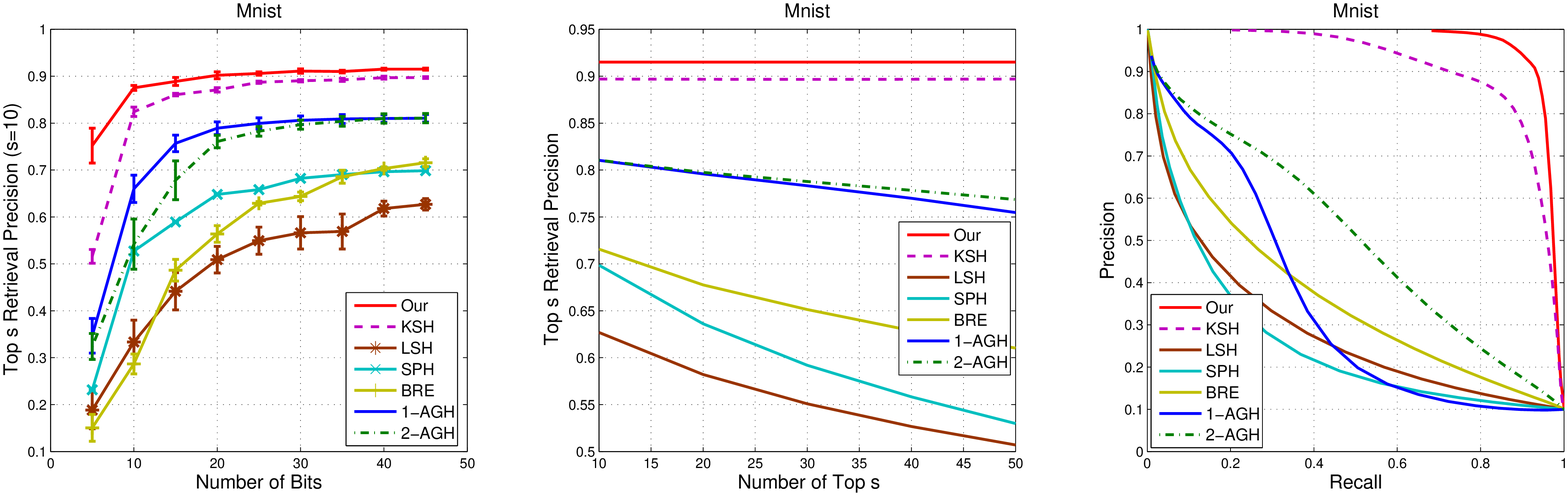}}
		\caption{The top $s$ retrieval results and Precision-Recall curve on \textit{Mnist} dataset over \ac{*SHL} and $6$ other hashing algorithms. (view in color)}
		\label{figure4}
	\end{center}
	\vskip -0.2in
\end{figure*} 

Five datasets were considered: \textit{Pendigits} and \textit{USPS} from the \textit{UCI Repository}, as well as \textit{Mnist}, \textit{PASCAL07} and \textit{CIFAR-10}. For \textit{Pendigits} ($10,992$ samples, $256$ features, $10$ classes), we randomly chose $3,000$ samples for training and the rest for testing; for \textit{USPS} ($9,298$ samples, $256$ features, $10$ classes), $3000$ were used for training and the remaining for testing; for \textit{Mnist} ($70,000$ samples, $784$ features, $10$ classes), $10,000$ for training and $60,000$ for testing; for \textit{CIFAR-10} ($60,000$ samples, $1,024$ features, $10$ classes), $10,000$ for training and the rest for testing; finally, for \textit{PASCAL07} ($6878$ samples, $1,024$ features after down-sampling the images, $10$ classes), $3,000$ for training and the rest for testing.

\begin{figure*}[htb]
	\vskip 0.2in
	\begin{center}
		\centerline{\includegraphics[width=\textwidth]{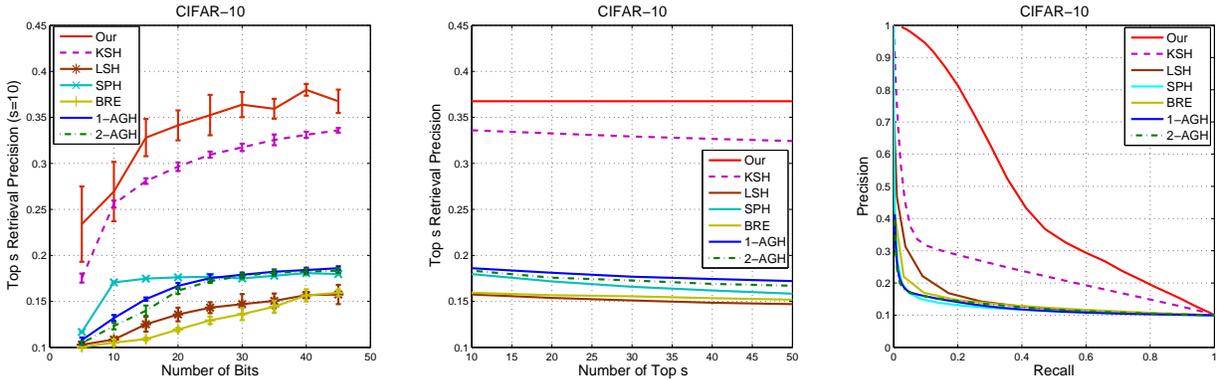}}
		\caption{The top $s$ retrieval results and Precision-Recall curve on \textit{CIFAR-10} dataset over \ac{*SHL} and $6$ other hashing algorithms. (view in color)}
		\label{figure5}
	\end{center}
	\vskip -0.2in
\end{figure*} 

\begin{figure*}[htb]
	\vskip 0.2in
	\begin{center}
		\centerline{\includegraphics[width=\textwidth]{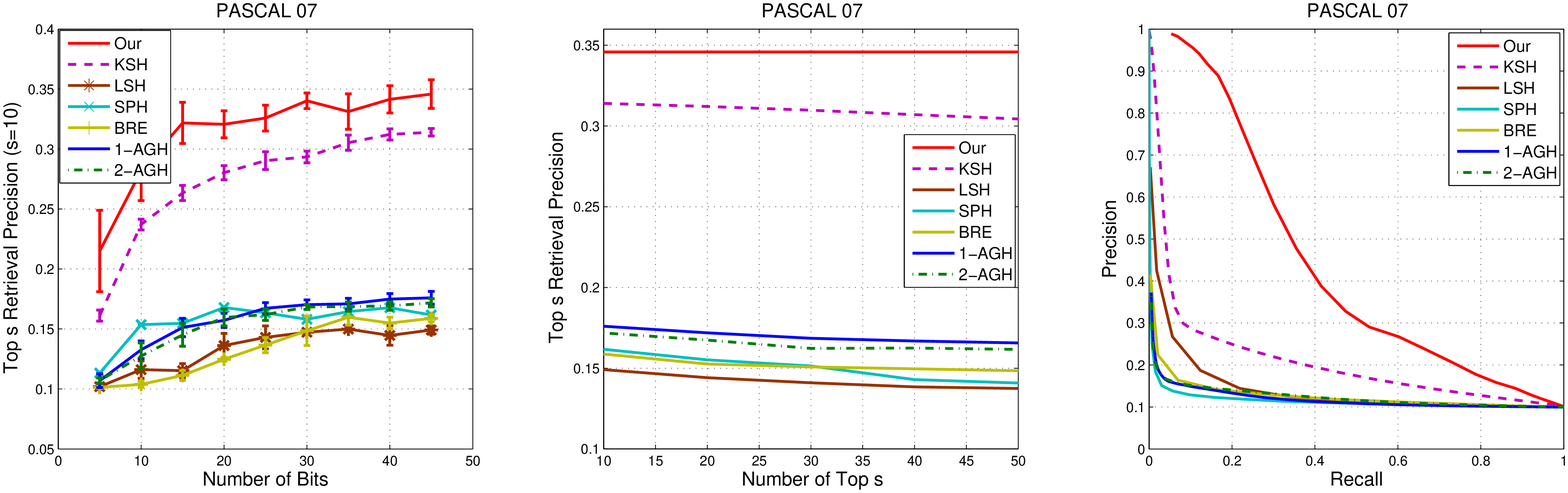}}
		\caption{The top $s$ retrieval results and Precision-Recall curve on \textit{PASCAL07} dataset over \ac{*SHL} and $6$ other hashing algorithms. (view in color)}
		\label{figure6}
	\end{center}
	\vskip -0.2in
\end{figure*} 

For all the algorithms used, average performances over $5$ runs are reported in terms of the following two criteria: (i) retrieval precision of $s$-closest hash codes of training samples; we used $s=\left\{ 10, 15, \ldots, 50 \right\}$. (ii) Precision-Recall (PR) curve, where retrieval precision and recall are computed for hash codes within a Hamming radius of $r \in \mathbb{N}_B$. 

The following \ac{*SHL} settings were used: \ac{SVM}'s parameter $C$ was set to $1000$; for \ac{MKL}, $11$ kernels were considered: $1$ normalized linear kernel, $1$ normalized polynomial kernel and $9$ Gaussian kernels. For the polynomial kernel, the bias was set to $1.0$ and its degree was chosen as $2$. For the bandwidth $\sigma$ of the Gaussian kernels the following values were used: $[2^{-7},2^{-5},2^{-3},2^{-1},1,2^1,2^3,2^5,2^7]$. Regarding the \ac{MKL} constraint set, a value of $p = 2$ was chosen. For the remaining approaches, namely KSH, SPH, AGH, BRE, parameter values were used according to recommendations found in their respective references. All obtained results are reported in \fref{figure2} through \fref{figure6}.

We clearly observe that \ac{*SHL} performs best among all the algorithms considered. For all the datasets, \ac{*SHL} achieves the highest top-$10$ retrieval precision. Especially for the non-digit datasets (\textit{CIFAR-10}, \textit{PASCAL07}), \ac{*SHL} achieves significantly better results. As for the PR-curve, \ac{*SHL} also yields the largest areas under the curve. Although noteworthy results were reported in \cite{Liu2012} for KSH, in our experiments \ac{*SHL} outperformed it across all datasets. Moreover, we observe that supervised hash learning algorithms, except BRE, perform better than unsupervised variants. BRE may need a longer bit length to achieve better performance as implied by \fref{figure2} and \fref{figure4}. Additionally, it is worth pointing out that \ac{*SHL} performed remarkably well for short big lengths across all datasets. 

It must be noted that AGH also yielded good results, compared with other unsupervised hashing algorithms, perhaps due to the anchor points it utilizes as side information to generate hash codes. With the exception of \ac{*SHL} and KSH, the remaining approaches exhibit poor performance for the non-digit datasets we considered. %(\textit{CIFAR-10} and \textit{PSACAL07}).

When varying the top-$s$ number between $10$ and $50$, once again with the exception of \ac{*SHL} and KSH, the performance of the remaining approaches deteriorated in terms of top-$s$ retrieval precision. KSH performs slightly worse, when $s$ increases, while \ac{*SHL}'s performance remains robust for \textit{CIFAR-10} and \textit{PSACAL07}. It is worth mentioning that the two-layer AGH exhibits better robustness than its single-layer version for datasets involving images of digits. Finally, \fref{figure7} shows some qualitative results for the \textit{CIFAR-10} dataset. In conclusion, in our experimentation, \ac{*SHL} exhibited superior performance for every code length we considered.

\begin{figure*}[htb]
	\vskip 0.2in
	\begin{center}
		\centerline{\includegraphics[width=\textwidth]{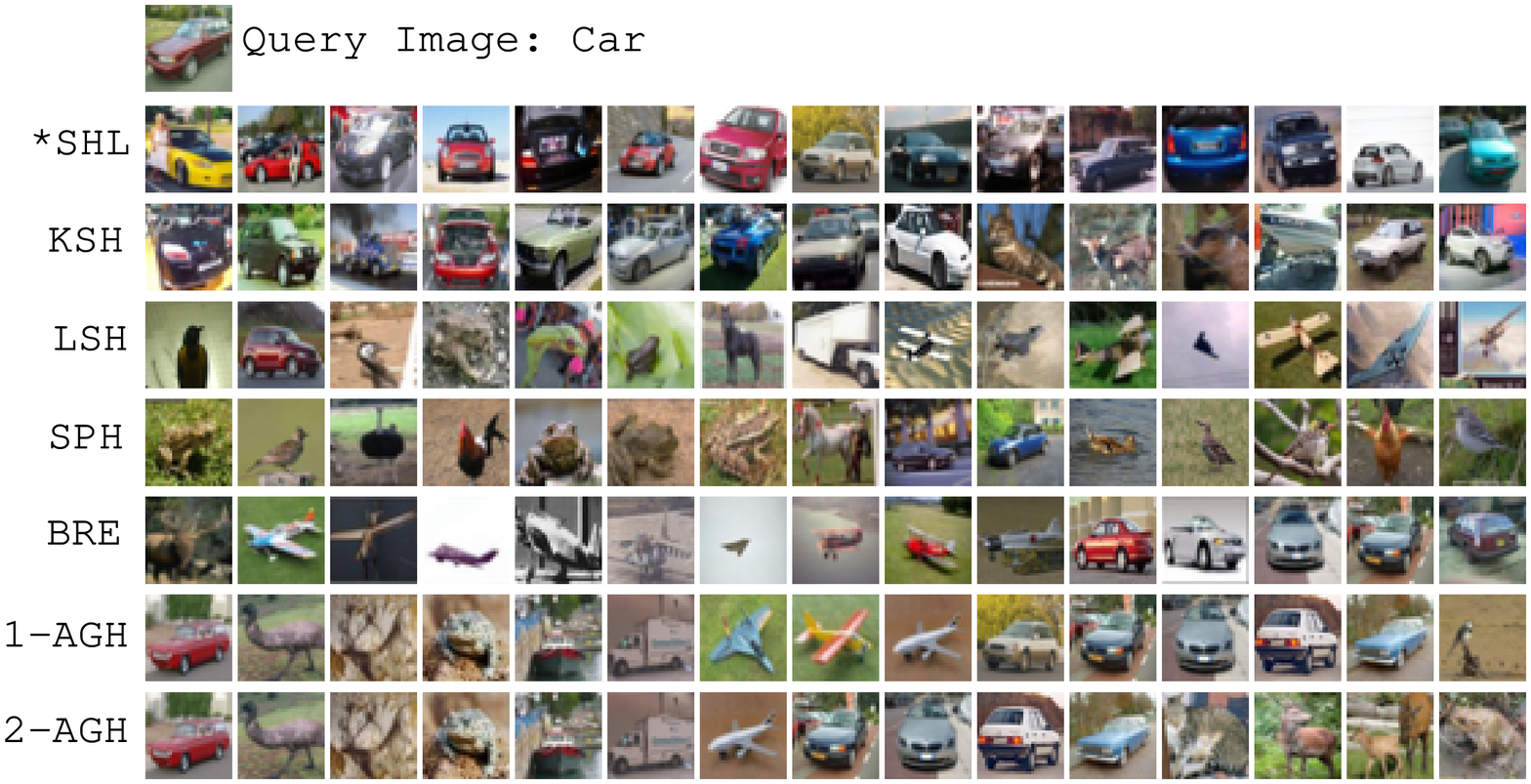}}
		\caption{Qualitative results on CIFAR-10. Query image is "Car". The remaining $15$ images for each row were retrieved using 45-bit binary codes generated by different hashing algorithms .}
		\label{figure7}
	\end{center}
	\vskip -0.2in
\end{figure*} 

\subsection{Transductive Hash Learning Results}
\label{toy}

\begin{figure}[ht]
\vskip 0.2in
\begin{center}
\centerline{\includegraphics[width=\textwidth]{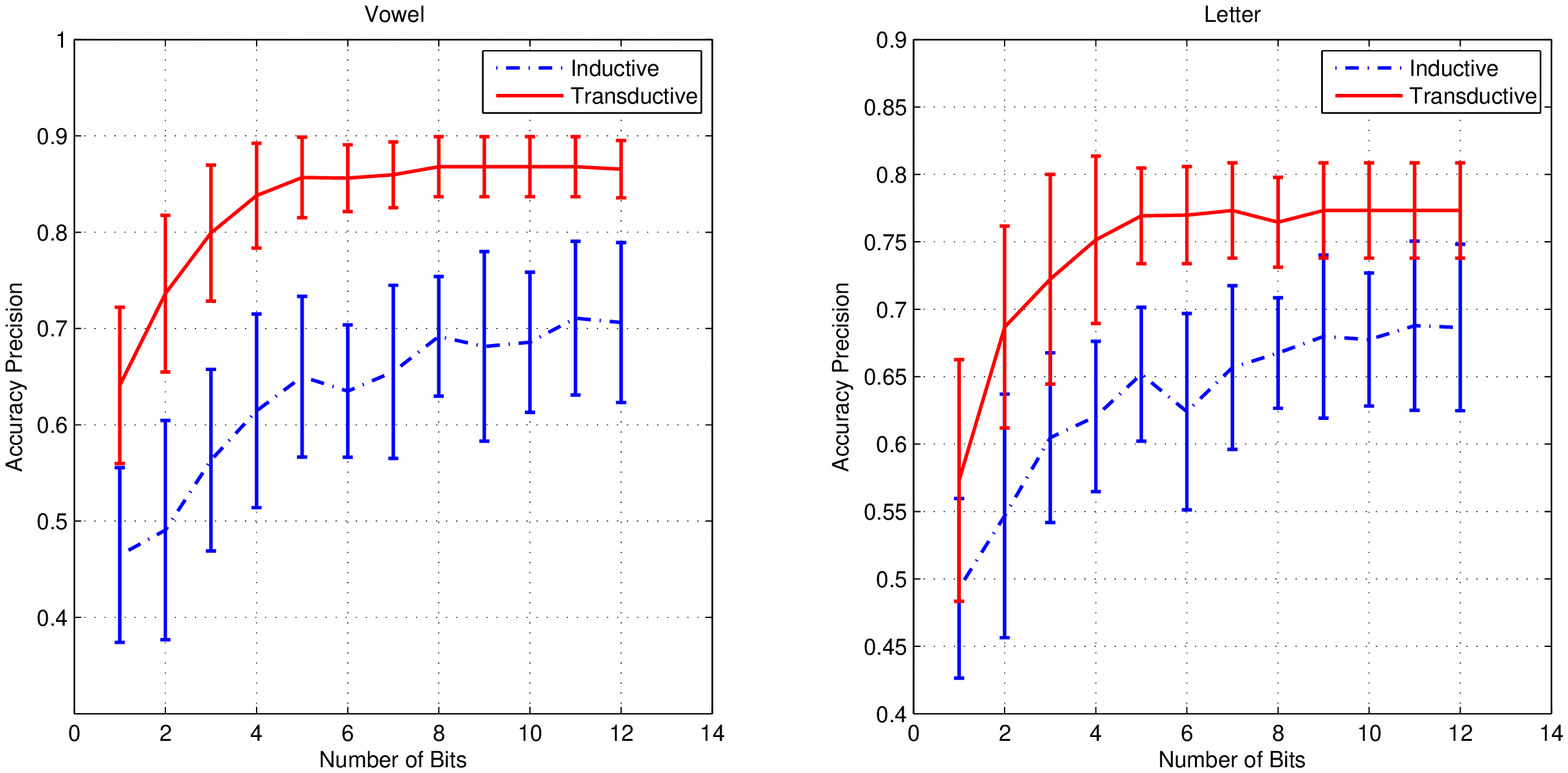}}
\caption{Accuracy results between Inductive and Transductive Learning.}
\label{figure1}
\end{center}
\vskip -0.2in
\end{figure}

As a proof of concept, in this section, we report a performance comparison of our framework, when used in an inductive versus a transductive \cite{Vapnik1998} mode. Note that, to the best of our knowledge, no other hash learning approaches to date accommodate transductive hash learning in a natural manner like \ac{*SHL}. For illustration purposes, we used the \textit{Vowel} and \textit{Letter} datasets. We randomly chose $330$ training and $220$ test samples for the \textit{Vowel} and $300$ training and $200$ test samples for the \textit{Letter}. Each scenario was run $20$ times and the code length ($B$) varied from $4$ to $15$ bits. The results are shown in \fref{figure1} and reveal the potential merits of the transductive \ac{*SHL} learning mode across a range of code lengths.

%% file: Conclusions.tex
\acresetall

\section{Conclusions}
\label{sec:Conclusions}

In this paper we considered a novel hash learning framework with two main advantages. First, its \ac{MM}/\ac{BCD} training algorithm is efficient and simple to implement. Secondly, this framework is able to address supervised, unsupervised and, even, semi-supervised learning tasks in a unified fashion. In order to show the merits of the method, we performed a series of experiments involving $5$ benchmark datasets. In these experiments, a comparison between \ac{*SHL} to $6$ other state-of-the-art hashing methods shows \ac{*SHL} to be highly competitive.

%% file: Acknowledgments.tex
\subsubsection{Acknowledgments}

Y. Huang was supported by a Trustee Fellowship provided by the Graduate College of the University of Central Florida. Additionally, M. Georgiopoulos acknowledges partial support from NSF grants No. 0806931, No. 0963146, No. 1200566, No. 1161228, and No. 1356233. Finally, G. C. Anagnostopoulos acknowledges partial support from NSF grant No. 1263011. Any opinions, findings, and conclusions or recommendations expressed in this material are those of the authors and do not necessarily reflect the views of the NSF.